\documentclass[final,3p,times,twocolumn]{elsarticle}

\journal{Neural Networks}

\usepackage{microtype}
\usepackage{amsmath}
\usepackage{amssymb}
\usepackage{bm}
\usepackage{graphicx}
\usepackage{subfigure}
\usepackage{epsfig}
\usepackage{epstopdf}
\usepackage{slashbox}
\usepackage{comment}
\usepackage{amsthm}
\usepackage{breqn}
\usepackage{booktabs} 
\usepackage [english]{babel}
\usepackage [autostyle, english = american]{csquotes}
\MakeOuterQuote{"}

\DeclareMathOperator*{\argmin}{arg\,min}

\newtheorem{theorem}{Theorem}
\usepackage{hyperref}
\usepackage{multirow}
\usepackage{xcolor,colortbl}
\usepackage{rotating}
\usepackage{algorithm}
\usepackage{algorithmic}

\usepackage{stackengine,scalerel}

\newcommand\overstar[1]{\ThisStyle{\ensurestackMath{%
			\setbox0=\hbox{$\SavedStyle#1$}%
			\stackengine{0pt}{\copy0}{\kern.2\ht0\smash{\SavedStyle*}}{O}{c}{F}{T}{S}}}}

\begin{document}

\begin{frontmatter}

\title{A General Framework for quantifying Aleatoric and Epistemic uncertainty in Graph Neural Networks}

\author[KSU]{Sai Munikoti\corref{CorrAuth}}
\author[KSU]{Deepesh Agarwal}
\author[ETH]{Laya Das}
\author[KSU]{Balasubramaniam Natarajan}

\address[KSU]{Department of Electrical and Computer Engineering, Kansas State University, Manhattan, Kansas 66506, USA}
\address[ETH]{Reliability and Risk Engineering Lab, ETH Zurich, 8092, Zurich, Switzerland.}

\cortext[CorrAuth]{Corresponding author. Email:saimunikoti@ksu.edu}

\begin{abstract} 
Graph Neural Networks (GNN) provide a powerful framework that elegantly integrates Graph theory with Machine learning for modeling and analysis of networked data. We consider the problem of quantifying the uncertainty in predictions of GNN stemming from modeling errors and measurement uncertainty. We consider aleatoric uncertainty in the form of probabilistic links and noise in feature vector of nodes, while epistemic uncertainty is incorporated via a probability distribution over the model parameters. We propose a unified approach to treat both sources of uncertainty in a Bayesian framework, where Assumed Density Filtering is used to quantify aleatoric uncertainty and Monte Carlo dropout captures uncertainty in model parameters. Finally, the two sources of uncertainty are aggregated to estimate the total uncertainty in predictions of a GNN. Results in the real-world datasets demonstrate that the Bayesian model performs at par with a frequentist model and provides additional information about predictions uncertainty that are sensitive to uncertainties in the data and model.
\end{abstract}



\end{frontmatter}

\section{Introduction}
Learning representations and relations over graphs finds applications in a wide range of networked systems such as social networks \cite{rassil2022augmented, ju2022ghnn}, biological networks \cite{liu2022minimum}, transportation networks \cite{ali2022exploiting} and communication networks \cite{jiang2021graph}. Uncertainty in measured quantities and imprecise information about the underlying structure and features of a network can pose a serious impediment to the efficiency of the learning process and quality of the resulting models. Uncertainty in estimated parameters and structure of trained model is a fundamental modeling challenge that imposes restrictions on the confidence of predictions. Learning representations in an uncertainty-aware manner is fundamental to producing robust models and reliable predictions. Models that do not account for these sources of uncertainty can be over-confident in their predictions \cite{jospin2020hands}. Moreover, neural network models are often prone to overfitting that limits their ability to generalise \cite{szegedy2013intriguing}. These factors can pose serious problems to effective utilization of the available information in the model building process as well as reliable interpretation of the model predictions under adverse situations \cite{munikoti2020bayesian,munikoti2022scalable} . Furthermore, quantification of uncertainty in predictions is also crucial for uncertainty sampling approach in active learning  \cite{gal2017deep, madhawa2020active}. Existing methods of analyzing graph structured data with Graph Neural Network (GNN) models are ill-equipped to handle uncertainty. We therefore consider the problem of quantifying the impact of different sources of uncertainty on the predictions of a GNN model.

Uncertainty quantification ubiquitously arises in modeling, and has been extensively addressed in the context of deep neural networks in computer vision \cite{maddox2019simple},  natural language processing \cite{liu2021qualifying} and robotics \cite{loquercio2020general}. It has been addressed in several ways including Laplace approximations, deep Gaussian processes and Bayesian methods \cite{abdar2021review}. A Bayesian approach to uncertainty quantification has the advantage of a principled treatment of different sources of uncertainty along with prior beliefs. Bayesian approaches have been very successful in handling uncertainty in neural networks for domain-invariant learning with uncertainty \cite{xiao2021bit}, addressing catastrophic forgetting \cite{chen2021overcoming}, imitation learning \cite{javed2021policy} and many more \cite{wang2020survey}.

In the context of GNN, uncertainty quantification and incorporation has received relatively less attention. Zhang et al \cite{zhang2019bayesian} propose a Bayesian framework where an observed graph is considered as a realization of a parametric family of random
graphs. This subsequently allows a joint inference of the graph structure and parameters of GNN from the data, resulting in a model that is resilient to noise and adversarial attacks. Pal et al. \cite{pal2019bayesian} follow up on the above idea, use a \textit{maximum a posteriori} (MAP) estimate of the graph and perform all learning tasks on the estimated graph instead of the measured graph. Since the MAP estimate corresponds to the mode of the probability density function of the graph structure, the authors consider it to inherently incorporate aleatoric uncertainty. The epistemic uncertainty in the resulting models are quantified with a Monte-Carlo dropout approach and the variance in the resulting predictions is reported \cite{zhang2019bayesian, pal2019bayesian}. While these methods adopt a Bayesian approach to mitigate the effect of uncertainty on the predictions, they consider the links and features of nodes to be deterministic and thus, do not consider the measurement uncertainty therein. In the presence of such measurement uncertainty, the GNN model propagates the true values as well as noise through all the layers, which in turn influences the model predictions. This phenomenon and its impact on the confidence of predictions is not captured in the literature. Moreover, the MAP estimates are obtained by processing the node features, which renders the approach sensitive to any uncertainty in the features. An explicit and systematic quantification of the uncertainty in predictions is also not provided.

\subsection{Contributions}
In this paper, we address the lack of systematic and explicit incorporation of different sources of uncertainty in GNNs within a Bayesian framework. We formally define the different sources of aleatoric and epistemic uncertainty in GNNs. Specifically, we consider the aleatoric uncertainty arising from (i) imprecise information about the graph structure via probabilistic links and (ii) measurement noise in feature vectors of nodes. We propagate the aleatoric uncertainty through the node embedding layers and classification layers of the GNN model via Assumed Density Filtering (ADF). We quantify the epistemic uncertainty arising from probabilistic parameters of the GNN model with Monte-Carlo sampling. The proposed framework exhibits the following advantages:
\begin{itemize}
\item Quantification of total uncertainty due to aleatoric and epistemic sources
\item Propagation of aleaotric uncertainty through all layers of a GNN model
\item Generic framework that can be applied at the time of fresh training, as well as to pre-trained networks without the need of a computationally expensive retraining process.
\end{itemize}
We present the different sources of uncertainty in GNNs and discuss the related literature on handling uncertainty in Section \ref{sec:background_related_work}. We discuss the specific problem definition targeted in this work and the proposed approach to quantify the total uncertainty in GNNs in Section \ref{sec:methodology}. Experiments and results are presented in Section \ref{sec:expts} and the article ends with concluding remarks in Section \ref{sec:conclusion}.

\section{Background and Related Work}
\label{sec:background_related_work}
Consider a graph $G$ represented as a tuple $G=(V,E)$ of a set of $n_V$ vertices/nodes $V$ and $n_E$ edges/links $E$. Each node $u_i$, $i\in[1,n_V]$ of the graph consists of $d$ features, represented as a vector $h_i=[h_{ij}]$, $j=1,2,3,...,d$. Each link $e_{ij}$, between the $i^{th}$ and $j^{th}$ nodes in the graph is associated with a weight $p_{ij}$, which signifies the strength of the link. We consider the weight to be normalised between $0$ and $1$ and interpret the weight as the probability of the corresponding link $e_{ij}$. In the following subsections, we discuss the different types and sources of uncertainty, and the existing literature on quantifying their impact on model predictions in vanilla neural networks and GNN.

\subsection{Aleatoric Uncertainty in Graph Neural Networks}
Aleatoric uncertainty refers to intrinsic randomness of the data due to noisy or inaccurate measurements. In the case of a GNN, the input data is in the form of graphs that model a real-world network. This data consists of the feature vectors $h_i$ of the nodes and probabilities/weights $p_{ij}$ of the links. The different sources of \textit{aleatoric uncertainty} in GNN are as follows:
\begin{enumerate}
    \item[AU1] Measurement uncertainty associated with feature vectors of nodes $h_{i}$, i.e., the measured feature vectors are considered as being the sum of true feature vectors ($h_i^*$) and measurement noise ($\epsilon_i$) as: $h_i=x_{i}+\epsilon_i$, with $\epsilon_i\sim p(\epsilon)$.
    \item[AU2] Structural uncertainty of the graph captured via probabilities of links $p_{ij}$, $i=1,2,3,...,n_V$, $j=1,2,3,...,n_V$.
\end{enumerate}
Uncertainty AU1 refers to the uncertainty in features of the nodes. The feature vectors can represent physical quantities such as the coverage of a cell tower in a communication network, or the functionalities of a protein in a biological network. For instance, every node (research article) in the citation dataset "Cora" is characterized by a feature vector of size $1433$, where each unit indicates the absence/presence of predefined words from the dictionary. These measured quantities can be uncertain in many scenarios because of the sensing process or imprecise information about the system. For example, in a protein network, the exact functionality of all proteins is rarely accurately known. This type of uncertainty propagates through the layers of the GNN and ultimately affects the model predictions.

Uncertainty AU2, on the other hand, represents the uncertainty in the existence of links. In the protein network scenario, the knowledge of interactions between different proteins and protein complexes is also highly uncertain, which results in probabilistic links between different nodes. Since the interactions are derived through noisy and error-prone
lab experiments, each link is associated with an uncertainty value \cite{asthana2004predicting}. For instance, a graph with $5$ nodes and binary weights of links has $2^5=32$ possible configurations. In a graph with continuous valued weights, these weights influence the extent to which information between any two nodes is exchanged and assimilated. This type of uncertainty therefore results in fundamentally different information exchange and processing through the graph.


A variety of Bayesian methods are used in the literature to handle aleatoric uncertainty in deep neural networks \cite{graves2011practical,blundell2015d,hernandez2015probabilistic, mae2021uncertainty}. However, there are very few works for GNN models. The authors in Zhang et al. \cite{zhang2019bayesian} propose a Bayesian framework using joint estimation of graph structure and GNN parameters. The authors make use of families of parametric random graphs to estimate the structure and parameters. This makes the approach sensitive to the choice of the random graph model and the extent to which the random network can accurately represent the characteristics of the true underlying network. As a result, inferences can be inconsistent for different problems and datasets. Another significant drawback of the technique is that the posterior inference of the graph is carried out solely conditioned on the observed graph. As a result, any information provided by the
node features and the training labels is completely disregarded. 
Therefore, Pal et al. \cite{pal2019bayesian} proposed an alternative approach which formulates the posterior inference of the graph in a non-parametric fashion, conditioned on the observed graph, features and training labels. Precisely, they obtain MAP estimate of graph, and conducted all the classification/regression tasks on this estimate. It is argued that MAP estimate handles aleatoric uncertainty of the input graph. However, the approach does not systematically define/quantify the sources of uncertainty and their impact on the predictions. Specifically, the uncertainties AU1 and AU2 are not considered in the framework. We address these shortcomings by explicitly incorporating AU1 and AU2 in our framework. Specifically, ADF is leveraged to propagate the aleatoric uncertainty from the input of the GNN to final node predictions through all the intermediate layers.

\subsection{Epistemic Uncertainty in Graph Neural Networks}
Epistemic uncertainty is the scientific uncertainty in the model that exists because of model in-competency to completely explain the underlying process. A GNN model $\Xi(\Theta)$ typically consists of several layers of complex aggregation and combination operations followed by feedforward processing. The different forms of \textit{epistemic uncertainty} in the context of GNNs are:
\begin{enumerate}
    \item[EU1] Parametric uncertainty in the GNN model, i.e., the parameters $\Theta$ of the GNN are assumed to be probabilistic with a probability density function $p(\Theta)$
    \item[EU2] Uncertainty in activation functions of neurons of the GNN model
\end{enumerate}

Uncertainty EU1 represents the uncertainty of the learnable parameters $\Theta$ of the GNN model, and is represented by by placing a distribution over the neural network weights. However, estimation of the posterior density function of the parameters $p(\Theta|D)$ given the data $D$ is mathematically intractable to compute for deep neural networks and is approximated by different methods. Among these methods, variational inference \cite{kingma2015variational}, and sampling-based approaches \cite{gal2016dropout} are the most effective ones.
Monte-Carlo sampling methods involve sampling the parameters from a distribution and are generally obtained using an ensemble of neural network predictions. The prediction ensemble could either be generated by differently trained networks \cite{lee2019ensemble}, or by using dropout at test-time \cite{gal2016dropout}.   

Similar to aleatoric uncertainty, the literature on handling epistemic uncertainty in GNN is limited. Zhang et al.\cite{zhang2019bayesian} and Pal et al. \cite{pal2019bayesian} are some of the few efforts that generated multiple Monte-Carlo samples by using dropout at test time. To address the problem of huge dependency of  Zhang et al.\cite{zhang2019bayesian} on the assumed random graph model, Hasan et al. \cite{hasanzadeh2020bayesian} introduces a stochastic regularization technique for GNN by adaptive connection sampling. Specifically, it adaptively learns the dropout rate for each layer of GNN. Akin to many of the aforementioned methods, we use the dropout-at-test approach to generate Monte-Carlo samples and estimate the epistemic uncertainty in predictions.

Uncertainty EU2 represents the randomness in the activation function of neurons in the GNN, and is not well considered in the literature. We also restrict our analysis to AU1, AU2 and EU1 in GNN and construct the total uncertainty from the contributions of these components.

\section{Methodology}
\label{sec:methodology}
A graph neural network typically involves two modules - node embedding and feedforward modules. The node embedding module performs aggregation and combination operations in the embedding layers of the model and produces a vector of \textit{node embeddings}. These operations capture the information propagation and processing phenomena in networked data. The feedforward module processes the node embeddings with non-linear transformations via feedforward layers and produces the final output. 
We next describe the proposed approach for quantification of the total uncertainty involved in GNN models.
\begin{figure*}[!h]
\centering
	\includegraphics[width=\textwidth]{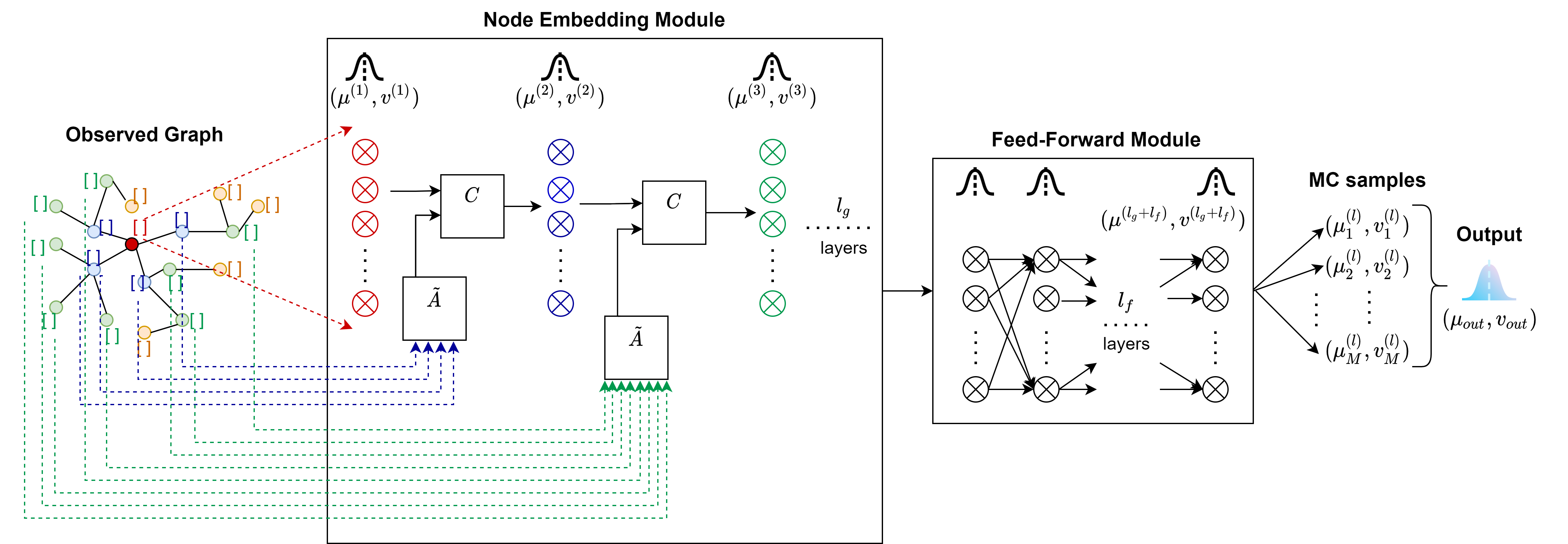}
	\caption{Proposed BGNN architecture for incorporating aleatoric and epistemic uncertainty in a GNN.}
	\label{fig:1}
\end{figure*}

\subsection{Problem definition}

We consider a network $G=(V,E)$ as described in Section 2 where nodes are associated with noisy feature vectors (AU1) and links are associated with probabilities (AU2) as follows:
\begin{enumerate}
    \item The noise in the features is considered to be Gaussian with zero mean and known variance. For a node $u$ in the network, we have:
    \begin{equation}
        \begin{split}
        h_u=h_u^*+\epsilon_u
        \\ \epsilon_u\sim\mathcal{N}(0,\Sigma_u)
        \end{split}
        \label{eq:feature_noise}
    \end{equation}
    where $h$ and $h^*$ are the measured and true feature vector respectively, $\epsilon_u$ is the noise in feature vector and $\Sigma_u$ is a diagonal matrix consisting of the known variances of noise in individual features. The noise in the features of different nodes are assumed to be uncorrelated, i.e., for any two nodes $u$ and $\nu$ in the network, we have:
    \begin{equation}
        \mathbb{E}[\epsilon_u\epsilon_\nu]=0\nonumber
    \end{equation}
    \item The probabilities $p_{u\nu}$ of links are assumed to be known \textit{a priori}.
    \item The learnable parameters $\Theta$ of the GNN are assumed to be random variables with an unknown probability density function:
    \begin{equation}
        \Theta\sim p_{\Theta}(\psi)\nonumber
    \end{equation}
\end{enumerate}
These sources of uncertainty result in a probabilistic propagation of the feature vectors through the model, and thus result in probabilistic outputs, i.e., $\hat{y}\sim p(y|h,p,\Theta)$. Obtaining the exact distribution of $\hat{y}$ is mathematically intractable. The problem of uncertainty quantification considered here is to systematically obtain the variance in the predictions because of the different sources of uncertainty. In the following subsections, we discuss how these effects are quantified in the proposed framework.

\subsection{Propagation of Aleatoric Uncertainty in GNN}
We propose a Bayesian approach to propagate the uncertainty in feature vectors through a GNN while also explicitly incorporating probabilistic links in the system. We achieve this with Assumed Density Filtering (ADF) and moment matching. We consider propagating and matching the mean and variance of the probability density function of outputs of all node embedding and feedforward layers of the model. While this can be achieved with existing approaches for feedforward layers, the following result formalises the result for node embedding layers.

\begin{theorem}
\label{thm:thm01}
The expected value ($\mu_u$) and variance ($v_u$) of the node embedding for node $u$, accounting for aleatoric uncertainty AU1 and AU2 with mean aggregation and linear activation functions are:
\begin{equation}
\mu_u^{(i)}=\theta_C^{(i)}\mu_u^{(i-1)}+\theta_A^{(i)}\frac{1}{|N(u)|}\sum_{\nu\in N(u)}p_{u\nu}\mu_{\nu}^{(i-1)}
\label{eq:0a}
\end{equation}
\begin{equation}
v_u^{{(i)}}=\theta_C^{(i)^2}v_u^{(i-1)}+\theta_A^{(i)^2}\frac{1}{|N(u)|D(u)}\sum_{\nu\in N(u)}p_{u\nu}^2v_{\nu}^{(i-1)}
\label{eq:0b}
\end{equation}
where the superscript $(i)$ represents the corresponding quantities of the $i^{th}$ node embedding layer; $\theta_C$ and $\theta_A$ represent the parameters of the combination and aggregation operations of GNN respectively; $N(u)$ and $D(u)$ represent the neighbourhood and degree of node $u$, respectively; $p_{u\nu}$ represents the probability of the link between nodes $u$ and $\nu$.
\end{theorem}

\begin{proof}
Consider the noisy feature vector for a node $u$ of the observed graph. This feature vector, along with the probabilistic graph structure are fed as input to the GNN. 
According to Eq. (\ref{eq:feature_noise}), feature vector can be expressed as:
\begin{equation}
h_u\sim\mathcal{N}(h_u^*,\Sigma)
\label{eq:1}
\end{equation}
This random variable is processed by the node embedding layers of the GNN model via aggregation and combination operations. The aggregation operation in the $i^{th}$ layer aggregates the embeddings of the neighbouring nodes $h_{N(u)}^{(i-1)}$ in the $(i-1)^{th}$ embedding layer and is equivalent to information collection operation from neighbours in the network. The combination operation combines the aggregated embeddings with the node embedding of the node $u$ and is equivalent to assimilating information from a network. The operation performed by the $i^{th}$ node embedding layer can be expressed as \cite{hamilton2017inductive}:
\begin{equation}
\begin{split}
h_{u}^{(i)} &= f^{(i)}\left(h_{u}^{(i-1)},h_{N(u)}^{(i-1)}\right) \\
&=  g\Bigg[\theta_{C}^{(i)}h_{u}^{(i-1)} +
\theta_{A}^{(i)}\tilde{A}\Big(h_{N(u)}^{(i-1)}\Big)\Bigg]
\end{split}
\label{eq:2a}
\end{equation}
where $g[\cdot]$ represents the activation function and $\tilde{A}(\cdot)$ denotes the aggregation operation. This operation is performed recursively $l_g$ number of times for a GNN with $l_g$ node embedding layers. The embeddings generated at the $i^{th}$ layer are dependent solely on the embeddings of the $(i-1)^{th}$ layer. As a result, joint density of all embeddings generated for a node $u$, i.e., $p(h_u^{(0:l_{g})})$ can be expressed as:
\begin{equation}
\begin{split}
p\Big(h_u^{(0:l_{g})}\Big) &= p\Big(h_u^{(0)}\Big) \prod_{i=1}^{l_{g}}p\Big(h_u^{(i)}|h_u^{(i-1)}\Big) \\
p\left(h_u^{(i)}|h_u^{(i-1)}\right) &= \delta \left[h_u^{(i)}- f^{i}\left(h_u^{(i-1)},h_{N(u)}^{(i-1)}\right)\right]
    \end{split}
    \label{eq:joint}
\end{equation}
where, $\delta[\cdot]$ is the Dirac delta function. This process is shown in Figure \ref{fig:1}. Propagating the uncertainty through the node embedding layers requires obtaining the joint density function described in Eq. \eqref{eq:joint}, which is mathematically intractable.

We employ ADF to approximate the joint density function. We choose ADF because of its low computational demand for a systematic propagation of uncertainty through all layers of a neural network \cite{boyen2013tractable, gast2018lightweight}. 
ADF approximates this intractable distribution as follows (we remove the subscript $u$ from the feature vector for the sake of brevity):
\begin{equation}
\begin{split}
p\left(h^{(0:l_{g})}\right) \approx q\left(h^{(0:l_{g})}\right) = q\left(h^{(0)}\right) \prod_{i=1}^{l_{g}}q\left(h^{(i)}\right) \\
    \end{split}
    \label{eq:4}
\end{equation}
ADF makes the first approximation by assuming that the probability density of the embeddings in the different layers are independent of each other. Furthermore, $q\left(h^{(i)}\right)$ is assumed to be Gaussian, so that we have:
\begin{equation}
\begin{split}
q\left(h^{(0)}\right) = p\left(h^{(0)}\right) \\
q\left(h^{(i)}\right) =  \prod_{j = 1}^{m_{i}}\mathcal{N}\left(\mu_{j}^{(i)}, v_{j}^{(i)}\right) 
    \end{split}
    \label{eq:5}
\end{equation}
where $m_{i}$ represents the size of the embedding vector at $i^{th}$ layer of the model, $\mu_j^{(i)}$ and $v_j^{(i)}$ are the mean and variance of the $j^{th}$ element of the embedding vector $h^{(i)}$. The approximate joint density function of all node embeddings till the $i^{th}$ layer can be expressed as:
\begin{equation}
\begin{split}
\tilde{p}\left(h^{(0:i)}\right) = p\left(h^{(i)} | h^{(i-1)}\right) \prod_{j=0}^{i-1}q\left(h^{(j)}\right) 
    \end{split}
    \label{eq:6}
\end{equation}
This step replaces the conditionals in Eq. \eqref{eq:joint} by the corresponding approximations from Eq. \eqref{eq:5} to obtain an approximate density $\tilde{p}\left(h^{(0:i)}\right)$. ADF then finds the best approximate distribution $q\left(h^{(0:i)}\right)$ by minimizing the KL divergence with $\tilde{p}\left(h^{(0:i)}\right)$ as:
\begin{equation}
\begin{split}
q\left(h^{(0:i)}\right) = \argmin_{\tilde{q}(h^{(0:i)})} KL \left(\tilde{q}\left(h^{(0:i)}\right) || \tilde{p}\left(h^{(0:i)}\right)\right)
    \end{split}
    \label{eq:7}
\end{equation}
This can be solved by matching the moments between the two distributions \cite{minka2001family}. Thus, any layer $h_u^{(i)}=f^{(i)}\left(h_u^{(i-1)},h_{N(u)}^{(i-1)}\right)$ can be converted into an uncertainty propagation layer by matching first two moments as:
\begin{equation}
\begin{split}
\mu_{u}^{(i)} &= \mathbb{E}_{q\left(h_{u}^{(i-1)}\right)}\left[f^{(i)}\left(h_{u}^{(i-1)},h_{N(u)}^{(i-1)}\right)\right]  \\
\end{split}
\label{eq:8}
\end{equation}
\begin{equation}
\begin{split}
v_{u}^{(i)} &= var_{q\left(h_{u}^{(i-1)}\right)}\left[f^{(i)}\left(h_u^{(i-1)},h_{N(u)}^{(i-1)}\right)\right]  \\
    \end{split}
    \label{eq:9}
\end{equation}
where, $\mathbb{E}$ and $var$ are the expectation and variance operators respectively. When the aggregation operation $\tilde{A}$ is the mean operator, and the activation function $g(\cdot)$ is linear, substituting Eq. \eqref{eq:2a} in Eqs. \eqref{eq:8} and \eqref{eq:9} yields Eqs. \eqref{eq:0a} and \eqref{eq:0b}, and hence proves Theorem \ref{thm:thm01}. 

This makes use of the following two identities: (1) Expectation of mean is equivalent to mean of expectations; and (2) Variance of means is the normalized form of mean of variances.

\end{proof}

Equations (\ref{eq:8}) and (\ref{eq:9}) can be determined analytically for most of the functions used in neural network such as ReLu, sigmoid, convolution, etc. For instance if the function $g$ is ReLu, then the modified mean and variance are \cite{frey1999variational}:
\begin{equation}
\begin{split}
\hat{\mu_{u}}^{(i)}\left(\mu_{u}^{(i)}, v_{u}^{(i)}\right)  = \mu_{u}^{(i)} \Phi \left(\frac{\mu_{u}^{(i)}}{\sigma_{u}^{(i)}}\right) + \sigma_{u}^{(i)} \phi\left(\frac{\mu_{u}^{(i)}}{\sigma_{u}^{(i)}}\right) 
\end{split}
    \label{eq:11a}
\end{equation}
\begin{equation}
\begin{split}
\hat{v_{u}}^{(i)}\left(v_{u}^{(i)}, v_{u}^{(i)}\right)  & = \left(\mu_{u}^{(i)} + v_{u}^{(i)}\right) \Phi \left(\frac{\mu_{u}^{(i)}}{\sigma_{u}^{(i)}}\right) + \\ 
& \hspace{0.6cm} \sigma_{u}^{(i)}\mu_{u}^{(i)} \phi\left(\frac{\mu_{u}^{(i)}}{\sigma_{u}^{(i)}}\right) - \left(\hat{\mu_{u}}^{(i)}\right)^2
    \end{split}
    \label{eq:11b}
\end{equation}
where, $\sigma_{u}^{(i)} = \sqrt{\biggl( v_{u}^{(i)} 
\biggr)}$, $\Phi$ and $\phi$ are the cumulative normal and standard normal distributions, respectively. Basically, Eqs. (\ref{eq:11a}) and (\ref{eq:11b}) are recursive formulae to compute mean $\left(\mu_{u}^{(i)}\right)$ and uncertainty $\left(v_{u}^{(i)}\right)$ of the embeddings, given the parameters of the  embedding distribution $q\left(h^{(i-1)}\right)$ in previous layer.

Typically, node embeddings from GNN are fed to feed-forward layers for classification/regression task. Therefore,  $\mu_{u}^{(l_{g})}$, $v_{u}^{(l_{g})}$  serve as an input to feed-forward layers, and mean and variance is propagated in a similar way as shown in Gast \& Roth \cite{gast2018lightweight}, Loquercio et al. \cite{loquercio2020general}. In a nutshell, ADF reshapes the forward pass of a GNN to generate not only output predictions $\mu_{u}^{(l)}$, but also their respective aleatoric uncertainties $v_{u}^{(l)}$. This is achieved by considering two values per dimension of both embeddings in GNN layers as well as neural units in feed-forward layers.  

\subsection{Propagation of Epistemic uncertainty in GNN}
Epistemic uncertainty, also known as model uncertainty refers to the model confidence on its prediction. This uncertainty arises because of the single adoption of weights out of many combinations that can attain same loss values on training data. This is usually captured by assuming a probability distribution for neural network weights rather than a scalar value. However, computation of this distribution $p(\omega|X,y)$ is usually intractable. Therefore, MC based approaches have been used to obtain different weight samples by using dropout at test time \cite{srivastava2014dropout, gal2016dropout, kendall2017uncertainties}. Specifically, in our case, epistemic uncertainty is the variance of $M$ MC samples obtained via different dropout masks as shown below:
\begin{equation}
\begin{split}
p(\omega|X,y) \approx q(\Theta; \phi ) &= \text{Bernoulli}(\Theta;\phi) \\
\sigma^{2}_{\text{model}} &= \frac{1}{T} \sum_{t=1}^{M}(y_{t}-\hat{y})^2
\end{split}
\label{eq:12}
\end{equation}
where, $\{ y_{t} \}_{t}^{M}$ is a set of M sampled outputs for different weight instances from the distribution $\omega^{t} \sim q(\omega, \phi)$ and $\hat{y}=\frac{1}{T}\sum_{t}y_{t}$. Authors in Gal et al.\cite{gal2016dropout} have shown that the optimal dropout rate $\phi$ for the computation of $\sigma_{\text{model}}$ is same as training dropout rate. 

\subsection{Total uncertainty in GNN}
Total variance of GNN predictions $y$ for a sample node with feature vector $X$ corrupted by noise variance $v^{0}$ can be written as:
\begin{equation}
\begin{split} 
\sigma_{tot} = \frac{1}{T} \sum_{t=1}^{T} v_{t}^{(L)} + \left(\mu_{t}^{(L)} - \hat{\mu}\right)^2, \\
\text{where,} \quad \hat{\mu} = \frac{1}{T} \sum_{t=1}^{T} \mu_{t}^{(L)}.
\label{eq:13}
\end{split}
\end{equation}

The first term $\left(v_{t}^{(L)}\right)$ in denotes aleatoric variance and the second term $\left(\mu_{t}^{(L)} - \hat{\mu}\right)^2$ represents the model uncertainty from \textit{M} MC predictions. $L=l_{g}+l_{f}$ is the total number of layers in GNN. Thus, the first part of total variance captures ensembles of propagated variance and the second part handles the ensembles of mean prediction, thereby, addressing both  aleatoric and epistemic uncertainty.     
The overall algorithm to compute total uncertainty can be summarized in following steps:
(i) Transform GNN into a bayesian network by associating mean and variance to each embedding vector and neuron unit; (ii) Obtain $M$ mean and variance predictions by forwarding $\left(X, v^{0}\right)$ to network with weights $\omega^{t}$ sampled from $q(\omega, \phi)$; (iii) Compute output predictions and its variances according to Eq. (\ref{eq:13}). 

\section{Experiments}
\label{sec:expts}
We apply our method to three standard datasets, namely, Cora, Amazon Computers and PubMed, with varying number of nodes, links, features and classes. The details of these datasets are as follows:
\begin{enumerate}
    \item \textbf{Cora:} This is a citation graph of scientific publications \cite{datasetsRef01}. Each node in this network represents a publication with binary features that indicate the presence/absence of different words from the dictionary. The edges represent citation links between different publications.  
    \item \textbf{Amazon Computers:} This is a subgraph of the Amazon co-purchase graph \cite{amazonCompRef01}, where each node represents a product and two nodes are connected by an edge if those two products are frequently bought together.
    \item \textbf{PubMed:} This is a network of scientific publications from PubMed database pertaining to diabetes \cite{datasetsRef01}. Each node is described by a TF/IDF weighted word vector from a dictionary and the undirected edges correspond to the citation links.
\end{enumerate}
These datasets encompass a wide range of graph-theoretic properties, i.e., numbers of nodes and links, dimension of feature vectors, average degree of nodes and average clustering coefficient. These properties of the datasets are summarized in Table \ref{tab:dataDescrip01}.
\begin{table}[h!]
\centering
\caption{Summary of the datasets considered. Avg. Deg.: Average Degree, Avg. CC: Average clustering coefficient.}
\label{tab:dataDescrip01}
\resizebox{\columnwidth}{!}{%
\begin{tabular}{|l|c|c|c|c|c|c|}
\hline
\textbf{Dataset} & \textbf{Nodes} & \textbf{Links} & \textbf{Features} & \textbf{Classes} & \textbf{Avg. Deg.} & \textbf{Avg. CC} \\ \hline
Cora & 2708 & 5429 & 1433 & 7 & 4.00 & 0.24 \\ \hline
\begin{tabular}[c]{@{}l@{}}Amazon\\ Computers\end{tabular} & 13752 & 491722 & 767 & 10 & 36.74 & 0.35 \\ \hline
PubMed & 19717 & 44338 & 500 & 3 & 6.34 & 0.06 \\ \hline
\end{tabular}
}
\end{table}

In order to test the effectiveness and generalizability of the method, we address the node classification task in the three networks with $3$, $7$ and $10$ classes. We compare the method with the state-of-the-art in the literature \cite{pal2019bayesian} and highlight the stark contrast in computational efficiency and quantification of uncertainty. We also demonstrate this ability of the model to capture the uncertainty via a sensitivity study.
It is important to note that the training of GNN is accomplished with standard cross entropy loss function that solely involves the mean prediction. As a part of future work, both mean and variance will be incorporated in the loss function via conditional log likelihood. This will allow using the information about estimated uncertainties for improving model performance/robustness, rather than just quantifying it.

\subsection{Baselines}
We compare the performance of the proposed method with Bayesian Graph Convolutional Network (BGCN) proposed by \cite{pal2019bayesian}. This work captures the aleatoric uncertainty through MAP estimation of the network structure and quantifies the epistemic uncertainty with Monte-Carlo sampling. However, uncertainties AU1 and AU2 are not explicitly incorporated. Moreover, the MAP estimation is dependent on accurate knowledge of node features, and hence the method is incapable of handling AU1. Finally, the approach is computationally expensive because of the MAP estimation step. These drawbacks are addressed with our method through explicit incorporation of uncertainty with the added advantage of reduced computational complexity.

\subsection{Experimental Setup}
\label{sec:appnd_B}
All training and evaluation experiments are performed on a computer with Intel i9-4820K processor running at 3.70GHz with 8 cores, one Nvidia RTX 2080 Ti GPU with 12 GB memory, and 64 GB RAM. The implementation is performed with the help of deep graph library in PyTorch. Evaluation is repeated $100$ times and average of metrics are reported for each dataset described in Section \ref{sec:main_resu}.

GNN is trained with the GraphSAGE algorithm \cite{hamilton2017inductive}. The detailed architecture of the GNN is as follows:
Depth i.e., no. of node embedding modules: 2; no. of neurons in $2$ layers: 64, 32; no. of Multi-layer perceptron (MLP) layers: 3; no. of neurons in MLP layers: 12,8,1; Activation function: Linear (except last layer with softmax); Aggregation function: Mean. 

The data is divided into training, validation and testing subsets by randomly masking $70\%$, $10\%$ and $20\%$ of the nodes respectively. Since the core task of the GNN is node classification, the node embeddings are concatenated with feedforward layers to provide class probabilities. The loss function is categorical cross entropy with ADAM optimizer.

The training is carried out in a mini-batch manner. The batch size is set to 50, the learning rate
is set to 0.001 with a dropout rate of 0.1. The models are trained 50 epochs in total.

\subsection{Sources of Uncertainty}
We introduce uncertainty in nodes feature (AU1) by adding Gaussian white noise to the true feature values with zero mean and a known variance as shown in Eq. (\ref{eq:feature_noise}). The proposed method is compared with the baseline with different levels of noise. The variance of noise is also varied to highlight the ability of the proposed method to capture the impact of this noise as it propagates through the model. We introduce uncertainty in links (AU2) with probability of nodes. These probabilities are not available in the datasets for all links in the networks. We therefore perform link prediction in a supervised manner following the approach presented by Zhang \& Chen \cite{zhang2018link} and obtain the probabilities of links. These predicted probabilities are then used for training Bayesian models in the proposed method. We introduce uncertainty in parameters of GNN (EU1) with a Bernoulli distribution of parameters according to Eq. (\ref{eq:12}).


\subsection{Results}
\label{sec:main_resu}

The results demonstrate the computational efficiency of the proposed approach, the ability to reflect different levels of uncertainty in predictions and generalizability of the approach. The source code to regenerate all the results can be accessed at this \href{https://github.com/saimunikoti/Uncertainty-BGNN}{link}. 

\subsubsection{Adequacy of BGNN}
We compare the proposed approach in this paper with Pal et al. \cite{pal2019bayesian}. To the best of authors' knowledge, Pal et al. \cite{pal2019bayesian} is the only work in the literature that deals with aleatoric uncertainty in GNN. A summary of results evaluated on Cora dataset based on 100 MC runs obtained via different dropout masks is presented in Table \ref{tab:perfComp_res01}. The results in Table \ref{tab:perfComp_res01} show that the proposed method yields higher classification accuracy as compared to the baseline BGCN in all the cases of input variance. In this work, we specify input variance as the percentage of mean features across all nodes in the dataset. The similar trends were observed in PubMed and Amazon Computers datasets as well. This demonstrates that learning MAP estimate of the input graph does not add much value in quantifying uncertainty related to noisy node feature vectors and link weights. On the other hand, the proposed method systematically propagates uncertainties through all the layers of GNN, as discussed in the forthcoming subsections.

\begin{table}[h!]
\centering
\caption{Performance comparison of BGCN and BGNN on Cora dataset (average of 100 MC runs). Standard deviation is shown underneath the average classification accuracy.}
\label{tab:perfComp_res01}
\resizebox{\columnwidth}{!}{%
\begin{tabular}{|c|cc|}
\hline
\multirow{1}{*}{\textbf{Input Variance}} &
\multicolumn{2}{c|}{\textbf{Classification Accuracy (\%)}} \\ \cline{2-3} 
 & \multicolumn{1}{c|}{\textbf{BGCN}} & \textbf{BGNN (ours)} \\ \hline
0.0\% & \multicolumn{1}{c|}{97.71 $\pm$ 0.0} & 97.96 $\pm$ 0.0 \\ \hline
2.5\% & \multicolumn{1}{c|}{89.25 $\pm$ 0.00357} & 90.21 $\pm$ 0.00311 \\ \hline
5.0\% & \multicolumn{1}{c|}{76.16 $\pm$ 0.00966} & 78.11 $\pm$ 0.0088 \\ \hline
12.0\% & \multicolumn{1}{c|}{58.60 $\pm$ 0.01153} & 61.22 $\pm$ 0.0108 \\ \hline
\end{tabular}
}
\end{table}

\subsubsection{Effectiveness of BGNN}
We demonstrate the effectiveness of the proposed approach by selecting a few samples from the datasets randomly, and examining the variations in class probabilities with changes in input variance. Table \ref{tab:longStud_res02} presents a summary of the results for three random samples, one from each class of the PubMed dataset. Firstly, it can be seen that the class probabilities corresponding to the true class decreases with increase in the levels of input variance. This is intuitive as a higher amount of input variance will introduce more uncertainty in the network, specifically in the node embeddings of GNN layers, thereby leading to reduction in class probabilities (i.e., moving towards a more uniform distribution). Secondly, in some cases like Sample ID 28, increasing levels of input variance may lead to mis-classification. This is because the prediction probability corresponding to the true class in no variance case is much lesser as compared to other samples. Finally, the total propagated variance at the output is observed to increase with increase in input variance across all samples. These examples illustrate the systematic propagation of uncertainties across all layers of GNN for different cases of input variance.   

\begin{table}[h!]
\centering
\caption{Longitudinal analysis of a few test samples selected at random from PubMed dataset. The values of prediction probabilities correspond to the class probabilities of true class for the respective samples. The values in bold indicate the misclassification at corresponding variance levels. The values of total propagated variance are the averages of total propagated variances across all classes.}
\label{tab:longStud_res02}
\resizebox{\columnwidth}{!}{%
\begin{tabular}{|c|c|cccccccc|}
\hline
\multirow{1}{*}{\textbf{\begin{tabular}[c]{@{}c@{}}Sample\\ ID\end{tabular}}} & \multirow{1}{*}{\textbf{\begin{tabular}[c]{@{}c@{}}True\\ Class\end{tabular}}} & \multicolumn{8}{c|}{\textbf{Metrics across different levels of input variance}} \\ \cline{3-10} 
 &  & \multicolumn{4}{c|}{\textbf{Prediction Probabilities}} & \multicolumn{4}{c|}{\textbf{Total propagated variance}} \\ \cline{3-10} 
 &  & \multicolumn{1}{c|}{\textbf{0.0\%}} & \multicolumn{1}{c|}{\textbf{2.5\%}} & \multicolumn{1}{c|}{\textbf{5.0\%}} & \multicolumn{1}{c|}{\textbf{12.0\%}} & \multicolumn{1}{c|}{\textbf{0.0\%}} & \multicolumn{1}{c|}{\textbf{2.5\%}} & \multicolumn{1}{c|}{\textbf{5.0\%}} & \textbf{12.0\%} \\ \hline
12 & 1 & \multicolumn{1}{c|}{0.923} & \multicolumn{1}{c|}{0.835} & \multicolumn{1}{c|}{0.767} & \multicolumn{1}{c|}{0.578} & \multicolumn{1}{c|}{0.0003} & \multicolumn{1}{c|}{0.0031} & \multicolumn{1}{c|}{0.0062} & 0.0166 \\ \hline
13 & 2 & \multicolumn{1}{c|}{0.775} & \multicolumn{1}{c|}{0.716} & \multicolumn{1}{c|}{0.689} & \multicolumn{1}{c|}{0.635} & \multicolumn{1}{c|}{0.0031} & \multicolumn{1}{c|}{0.0081} & \multicolumn{1}{c|}{0.0119} & 0.0218 \\ \hline
28 & 3 & \multicolumn{1}{c|}{0.551} & \multicolumn{1}{c|}{\textbf{0.353}} & \multicolumn{1}{c|}{\textbf{0.262}} & \multicolumn{1}{c|}{\textbf{0.131}} & \multicolumn{1}{c|}{0.0023} & \multicolumn{1}{c|}{0.0034} & \multicolumn{1}{c|}{0.0044} & 0.0063 \\ \hline
\end{tabular}
}
\end{table}

\subsubsection{Generalizability of BGNN}
The generalizability of proposed approach is examined by performing experiments over graphs of different sizes and characteristics. The model performance is evaluated using classification accuracy. Owing to the absence of ground truth for variance assessment, it is evaluated using average per-class negative log likelihood (NLL) \cite{kendall2017uncertainties,gast2018lightweight,loquercio2020general}. The value of NLL for a specific class is evaluated as:
\begin{equation}
\begin{split}
NLL = \frac{1}{2} \log(\sigma_{\text{tot}}) + \frac{1}{2\sigma_{\text{tot}}}(y-\hat{y})^2
\label{eq:14}
\end{split}
\end{equation}
where, $\hat{y}$ is the mean prediction of class probabilities across $100$ MC runs and $y$ is a 0/1 value indicating whether the given node belongs to a specific class. $\sigma_{\text{tot}}$ is the total variance comprising of propagated input variance and that due to model uncertainty. Tables \ref{tab:Cora_res_01}, \ref{tab:amazcomp_res_01} and \ref{tab:pubmed_res_01} depict the metrics values for Cora, Amazon Computers and PubMed datasets, respectively. These values are obtained based on the average of $100$ MC runs, and mean of per-class NLL is reported. It can be observed for all the datsets that the mean classification accuracy of the model decreases with the increase in the input variance. It is intuitive in a sense that as variance in the input feature vector increases, it consistently becomes hard for the model to uniquely identify nodes with the node embeddings and thereby their labels. This idea is also reinforced by the increasing values of prediction loss observed with increase in input variance. 

The total variance propagated at output ($\sigma_{\text{tot}}$) in all the cases is also indicated in Tables \ref{tab:Cora_res_01} - \ref{tab:pubmed_res_01}. It can be seen that $\sigma_{\text{tot}}$ lies between $0$ and $1$ in all the cases of datsets considered in this work. Therefore, the first term in right hand side (RHS) of eq. (\ref{eq:14}) will always be negative. If the values of $\sigma_{\text{tot}}$ are relatively higher, as in Cora (Table \ref{tab:Cora_res_01}) and Amazon Computers (Table \ref{tab:amazcomp_res_01}) datasets, the values of first term in eq. (\ref{eq:14}) dominate, the second term will not be positive enough and consequently, the overall NLL values turn out to be negative. In these cases, the NLL values increase with increase in input variance as $\log(\sigma_{\text{tot}})$ is a monotonically increasing function. This is clearly evident from the NLL values in Tables \ref{tab:Cora_res_01} and \ref{tab:amazcomp_res_01}. On the other hand, if the values of $\sigma_{\text{tot}}$ are relatively lower, as in PubMed (Table \ref{tab:pubmed_res_01}) dataset, the values of second term in eq. (\ref{eq:14}) dominate and the overall NLL values are positive. In such cases, the NLL values decrease with increase in input variance, as observed in Table \ref{tab:pubmed_res_01}. Thus, NLL demonstrates the high quality estimates of uncertainty without changing or re-training the GNN.

\begin{table}[h!]
\centering
\caption{Results for Cora dataset (average of $100$ MC runs). Input variance is specified as percentage of mean features across all nodes in the dataset.}
\label{tab:Cora_res_01}
\resizebox{\columnwidth}{!}{%
\begin{tabular}{|c|c|c|c|c|}
\hline
\textbf{\begin{tabular}[c]{@{}c@{}}Input \\ Variance\end{tabular}} & \textbf{\begin{tabular}[c]{@{}c@{}}Classification \\ Accuracy\end{tabular}} & \textbf{\begin{tabular}[c]{@{}c@{}}Prediction \\ Loss\end{tabular}} & \textbf{\begin{tabular}[c]{@{}c@{}}Avg. per \\ class NLL\end{tabular}} & \textbf{\begin{tabular}[c]{@{}c@{}}Variance \\ propagated \\ at Output\end{tabular}} \\ \hline
0.0\% & 97.96\% & 0.19 & - & - \\ \hline
2.5\% & 90.21\% & 0.45 & -0.98 & 0.12 \\ \hline
5.0\% & 78.11\% & 0.76 & -0.65 & 0.24 \\ \hline
12.0\% & 61.22\% & 1.64 & -0.23 & 0.57 \\ \hline
\end{tabular}
}
\end{table}

\begin{table}[h!]
\centering
\caption{Results for Amazon-copurchase computer dataset}
\label{tab:amazcomp_res_01}
\resizebox{\columnwidth}{!}{%
\begin{tabular}{|c|c|c|c|c|}
\hline
\textbf{\begin{tabular}[c]{@{}c@{}}Input \\ Variance\end{tabular}} & \textbf{\begin{tabular}[c]{@{}c@{}}Classification \\ Accuracy\end{tabular}} & \textbf{\begin{tabular}[c]{@{}c@{}}Prediction \\ Loss\end{tabular}} & \textbf{\begin{tabular}[c]{@{}c@{}}Avg. per \\ class NLL\end{tabular}} & \textbf{\begin{tabular}[c]{@{}c@{}}Variance \\ propagated \\ at Output\end{tabular}} \\ \hline
0.0\% & 82.56\% & 0.599 & - & - \\ \hline
2.5\% & 81.88 \% & 0.624 & -1.58 & 0.0137 \\ \hline
5.0\% & 80.79 \% & 0.6435 & -1.47 & 0.0246 \\ \hline
12.0\% & 78.26 \% & 0.6938 & -1.24 & 0.0549 \\ \hline
\end{tabular}
}
\end{table}

\begin{table}[h!]
\centering
\caption{Results for PubMed dataset.}
\label{tab:pubmed_res_01}
\resizebox{\columnwidth}{!}{%
\begin{tabular}{|c|c|c|c|c|}
\hline
\textbf{\begin{tabular}[c]{@{}c@{}}Input \\ Variance\end{tabular}} & \textbf{\begin{tabular}[c]{@{}c@{}}Classification \\ Accuracy\end{tabular}} & \textbf{\begin{tabular}[c]{@{}c@{}}Prediction \\ Loss\end{tabular}} & \textbf{\begin{tabular}[c]{@{}c@{}}Avg. per \\ class NLL\end{tabular}} & \textbf{\begin{tabular}[c]{@{}c@{}}Variance \\ propagated \\ at Output\end{tabular}} \\ \hline
0.0\% & 84.00\% & 0.40 & - & - \\ \hline
2.5\% & 82.83\% & 0.46 & 10.75 & 0.0029 \\ \hline
5.0\% & 80.70\% & 0.52 & 6.83 & 0.0046 \\ \hline
12.0\% & 76.03\% & 0.70 & 3.90 & 0.0092 \\ \hline
\end{tabular}
}
\end{table}

\section{Conclusions}
\label{sec:conclusion}
In this paper, we propose a generic framework for incorporating aleatoric and epistemic uncertainty in GNN.
The aleatoric uncertainty arising from imprecise information about graph structure (probabilistic links) and node features is propagated via ADF. 
On the other hand, epistemic uncertainity arising from the probabilistic parameters of GNN model is quatified through MC sampling.
The proposed method, BGNN, systematically propagates these uncertainites through the layers of GNN to final predictions without the need of retraining.
Furthermore, this method is agnostic to network architecture, algorithm and the learning tasks.
Experimental results show that BGNN achieves superior performance in quantifying unceratinites for different levels of input noise across several types of graphs.
The future extension of this work shall focus on leveraging this knowledge of propagated uncertainty to modify training objectives and thereby, improve model performance and robustness.

\section*{Acknowledgement}
This material is based upon work supported by National Science Foundation under award number $1855216$.

\bibliographystyle{elsarticle-num}
\bibliography{reference}

\end{document}